%% file: DBSCANPP.tex
\documentclass{article}

\usepackage{hyperref}

\usepackage[accepted]{icml2019}

\usepackage{graphicx}
\usepackage[utf8]{inputenc} 
\usepackage[T1]{fontenc}    
\usepackage{hyperref}       
\usepackage{url}            
\usepackage{booktabs}       
\usepackage{amsfonts}       
\usepackage{nicefrac}       
\usepackage{microtype}      
\usepackage{amsmath}
\usepackage{wrapfig}
\usepackage{amsthm}
\usepackage{amssymb}
\usepackage{xcolor}
\definecolor{darkorchid}{rgb}{1.0, 0.58, 0.16}
\definecolor{darkturquoise}{rgb}{0.32, 0.63, 0.34}

\newtheorem{lemma}{Lemma}
\newtheorem{theorem}{Theorem}
\newtheorem{corollary}{Corollary}
\newtheorem{definition}{Definition}
\newtheorem{assumption}{Assumption}
\newtheorem{remark}{Remark}

\theoremstyle{definition}

\DeclareMathOperator*{\argmax}{argmax}

\icmltitlerunning{DBSCAN++: Towards fast and scalable density clustering}

\begin{document}

\twocolumn[
\icmltitle{DBSCAN++: Towards fast and scalable density clustering}



\icmlsetsymbol{equal}{*}

\begin{icmlauthorlist}
\icmlauthor{Jennifer Jang}{uber}
\icmlauthor{Heinrich Jiang}{goo}
\end{icmlauthorlist}

\icmlaffiliation{uber}{Uber}
\icmlaffiliation{goo}{Google Research}

\icmlcorrespondingauthor{Jennifer Jang}{j.jang42@gmail.com}

\icmlkeywords{Machine Learning, ICML}

\vskip 0.3in
]



\printAffiliationsAndNotice{} 

\begin{abstract}
DBSCAN is a classical density-based clustering procedure with tremendous practical relevance. However, DBSCAN implicitly needs to compute the empirical density for each sample point, leading to a quadratic worst-case time complexity, which is too slow on large datasets. We propose DBSCAN++, a simple modification of DBSCAN which only requires computing the densities for a chosen subset of points. We show empirically that, compared to traditional DBSCAN, DBSCAN++ can provide not only competitive performance but also added robustness in the bandwidth hyperparameter while taking a fraction of the runtime. We also present statistical consistency guarantees showing the trade-off between computational cost and estimation rates. Surprisingly, up to a certain point, we can enjoy the same estimation rates while lowering computational cost, showing that DBSCAN++ is a sub-quadratic algorithm that attains minimax optimal rates for level-set estimation, a quality that may be of independent interest.
\end{abstract}

\section{Introduction}
\input{Introduction}

\section{Related Works}
\input{RelatedWorks}

\section{Algorithm}
\input{Algorithm}

\section{Theoretical Analysis}
\input{Theory.tex}

\section{Experiments}
\input{Experiments}


\section{Conclusion}
\input{Conclusion.tex}

\clearpage
{
\bibliography{paper}
\bibliographystyle{plainnat}
}

\clearpage
{
\appendix
\onecolumn
{\Large \bf Appendix}
\input{Appendix.tex}
}
\end{document}

%% file: Introduction.tex
Density-based clustering algorithms such as Mean Shift \citep{cheng1995mean} and DBSCAN \citep{ester1996density} have made a large impact on a wide range of areas in data analysis, including outlier detection, computer vision, and medical imaging. As data volumes rise, non-parametric unsupervised procedures are becoming ever more important in understanding large datasets. Thus, there is an increasing need to establish more efficient versions of these algorithms. In this paper, we focus on improving the classical DBSCAN procedure. 

It was long believed that DBSCAN had a runtime of $O(n\log{n})$ until it was proven to be $O(n^2)$ in the worst case by \citet{gan2015dbscan}. They showed that while DBSCAN can run in $O(n\log{n})$ when the dimension is at most $2$, it quickly starts to exhibit quadratic behavior in high dimensions and/or when $n$ becomes large. In fact, we show in Figure~\ref{figure:3gaussian} that even with a simple mixture of 3-dimensional Gaussians, DBSCAN already starts to show quadratic behavior. 

The quadratic runtime for these density-based procedures can be seen from the fact that they implicitly must compute density estimates for each data point, which is linear time in the worst case for each query. In the case of DBSCAN, such queries are proximity-based. There has been much work done in using space-partitioning data structures such as KD-Trees \citep{bentley1975multidimensional} and Cover Trees \citep{beygelzimer2006cover} to improve query times, but these structures are all still linear in the worst-case. Another line of work that has had practical success is in approximate nearest neighbor methods (e.g. \citet{indyk1998approximate,datar2004locality}) which have sub-linear queries, but such methods come with few approximation guarantees.

DBSCAN proceeds by computing the empirical densities for each sample point and then designating points whose densities are above a threshold as {\it core-points}. Then, a neighborhood graph of the core-points is constructed, and the clusters are assigned based on the connected components.

In this paper, we present DBSCAN++, a step towards a fast and scalable DBSCAN. DBSCAN++ is based on the observation that we only need to compute the density estimates for a subset $m$ of the $n$ data points, where $m$ can be much smaller than $n$, in order to cluster properly. To choose these $m$ points, we provide two simple strategies: uniform and greedy $K$-center-based sampling. The resulting procedure has $O(mn)$ worst-case runtime.  

We show that with this modification, we still maintain statistical consistency guarantees. We show the trade-off between computational cost and estimation rates. Interestingly, up to a certain point, we can enjoy the same minimax-optimal estimation rates attained by DBSCAN while making $m$ (instead of the larger $n$) empirical density queries, thus leading to a sub-quadratic procedure. In some cases, we saw that our method of limiting the number of core points can act as a regularization, thus reducing the sensitivity of classical DBSCAN to its parameters.

We show on both simulated datasets and real datasets that DBSCAN++ runs in a fraction of the time compared to DBSCAN, while giving competitive performance and consistently producing more robust clustering scores across hyperparameter settings.

%% file: RelatedWorks.tex
\begin{figure}
\includegraphics[width=0.48\textwidth]{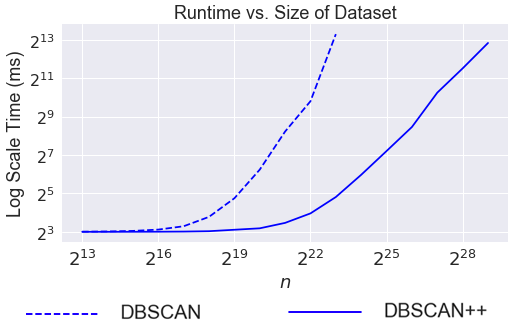}
\vspace{-0.5cm}
\caption{\label{figure:3gaussian}\textit{Runtime (seconds) vs dataset size to cluster a mixture of four $3$-dimensional Gaussians.} Using Gaussian mixtures, we see that DBSCAN starts to show quadratic behavior as the dataset gets large. After $10^6$ points, DBSCAN ran too slowly and was terminated after 3 hours. This is with only $3$ dimensions.}
\vspace{-0.5cm}
\end{figure}


There has been much work done on finding faster variants of DBSCAN. We can only highlight some of these works here. One approach is to speed up the nearest neighbor queries that DBSCAN uses \cite{huang2009grid,vijayalaksmi2012fast,kumar2016fast}. Another approach is to find a set of "leader" points that still preserve the structure of the original data set and then identify clusters based on the clustering of these "leader" points \cite{geng2000fast,viswanath2006dbscan,viswanath2009rough}. Our approach  of finding core points is similar but is simpler and comes with theoretical guarantees.  \citet{liu2006fast} modified DBSCAN by selecting clustering seeds among the unlabeled core points in an orderly manner in order to reduce computation time in regions that have already been clustered. Other heuristics include \citep{borah2004improved,zhou2000fdbscan,patwary2012new,kryszkiewicz2010ti}. 

There are also numerous approaches based on parallel computing such as \cite{xu1999fast,zhou2000approaches,arlia2001experiments,brecheisen2006parallel,chen2010parallel,patwary2012new,gotz2015hpdbscan} including map-reduce based approaches \cite{fu2011research,he2011mr,dai2012efficient,noticewala2014mr}. Then there are distributed approaches to DBSCAN where data is partitioned across different locations and there may be communication cost constraints \cite{januzaj2004scalable,januzaj2004dbdc,liu2012privacy,neto2015g2p,lulli2016ng}.  It is also worth mentioning \citet{andrade2013g}, who presented a GPU implementation of DBSCAN that can be over 100x faster than sequential DBSCAN. In this paper, we assume a single processor although extending our approach to the parallel or distributed settings could be a future research direction.

We now discuss the theoretical work done for DBSCAN. Despite the practical significance of DBSCAN, its statistical properties has only been explored recently \cite{sriperumbudur2012consistency,jiang2017density,wang2017optimal,steinwart2017adaptive}. Such analyses make use of recent developments in topological data analysis to show that DBSCAN estimates the connected components of a level-set of the underlying density.

It turns out there has been a long history in estimating the level-sets of the density function \cite{hartigan1975clustering,tsybakov1997nonparametric,singh2009adaptive,rigollet2009optimal,rinaldo2010generalized,chaudhuri2010rates,steinwart2011adaptive,balakrishnan2013cluster,chaudhuri2014consistent,jiang2017uniform,chen2017density}. However, most of these methods have little practical value (some are unimplementable), and DBSCAN is one of the only practical methods that is able to attain the strongest guarantees, including finite-sample Hausdorff minimax optimal rates. In fact the only previous method that was shown to attain such guarantees was the impractical histogram-based method of \citet{singh2009adaptive}, until \citet{jiang2017density} showed that DBSCAN attained almost identical guarantees.  This paper shows that DBSCAN++ can attain similar guarantees while being sub-quadratic in computational complexity as well as the precise trade-off in estimation rates for further computational speedup.

%% file: Algorithm.tex
We have $n$ i.i.d. samples $X = \{x_1,...,x_n\}$ drawn from a distribution $\mathcal{F}$ over $\mathbb{R}^D$. We now define {\it core-points}, which are essentially points with high empirical density defined with respect to the two hyperparameters of DBSCAN, $\text{minPts}$ and $\varepsilon$. The latter is also known as the {\it bandwidth}.
\begin{definition}
Let $\varepsilon > 0$ and $\text{minPts}$ be a positive integer. Then $x \in X$ is a core-point if $|B(x, \varepsilon) \cap X| \ge \text{minPts}$, where $B(x, \varepsilon) := \{x' : |x - x'| \le \varepsilon\}$. 
\end{definition}
In other words, a core-point is a sample point that has at least $\text{minPts}$ sample points within its $\varepsilon$-radius neighborhood.


DBSCAN \citep{ester1996density} is shown as Algorithm~\ref{fig:dbscan}, which is in a more concise but equivalent form to the original version (see  \citet{jiang2017density}). It creates a graph $G$ with core-points as vertices and edges connecting core points, which are distance at most $\varepsilon$ apart. The final clusters are represented by the connected components in this graph along with non-core-points that are close to such a connected component. The remaining points are designated as noise points and are left unclustered. Noise points can be seen as outliers.

\begin{algorithm}[H]
\caption{\label{fig:dbscan}DBSCAN}
\begin{algorithmic}
\STATE {\bf Inputs}: $X$, $\varepsilon$, minPts
\STATE $C \gets$ core-points in $X$ given $\varepsilon$ and minPts
\STATE $\textit{G} \gets$ initialize empty graph
\FOR {$c \in C$}
    \STATE Add an edge (and possibly a vertex or vertices) in $G$ from $c$ to all points in $X \cap B(c, \varepsilon)$
\ENDFOR
\STATE {\bf return} connected components of $G$.
\end{algorithmic}
\end{algorithm}

\begin{figure}[H]
\begin{center}
\includegraphics[width=0.48\textwidth]{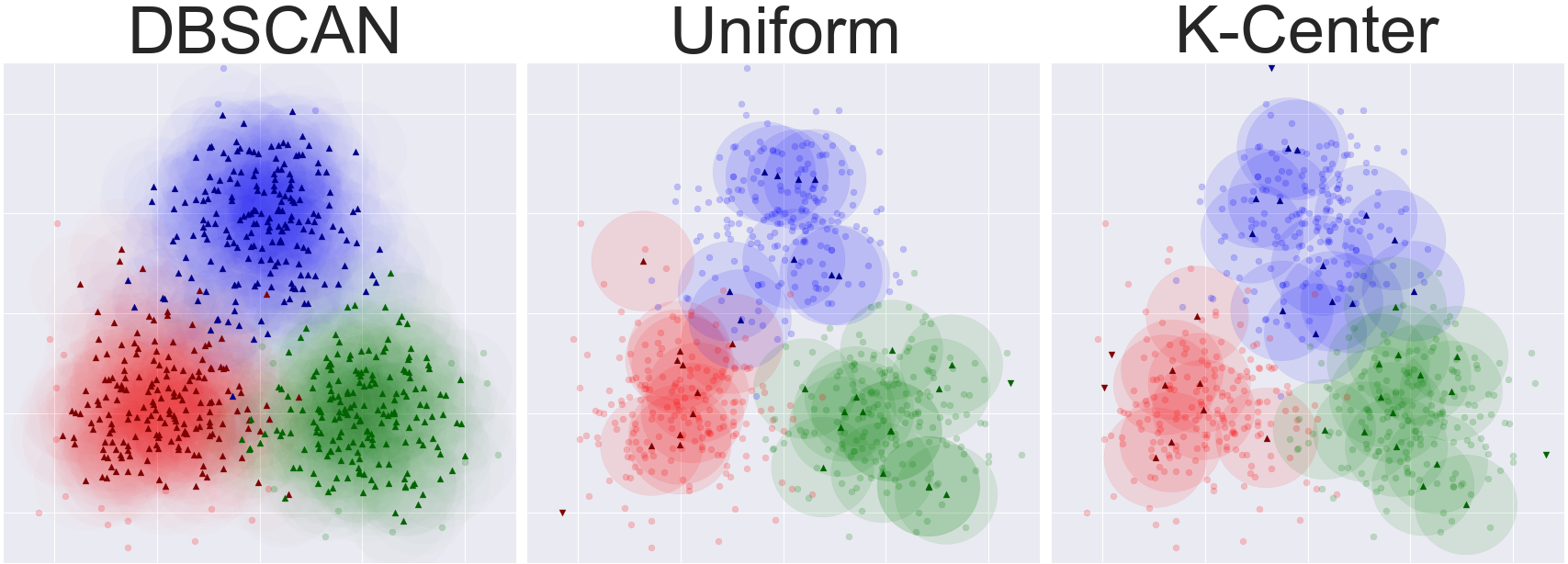}
\end{center}
\caption{\label{fig:coverage} \textit{Core-points from a mixture of three 2D Gaussians.} Each point marked with a triangle represents a core-point and the shaded area its $\varepsilon$-neighborhood. The total $\varepsilon$-radii area of DBSCAN++ core-points provides adequate coverage of the dataset. The $K$-center initialization produces an even more efficient covering. The points that are not covered will be designated as outliers. This illustrates that a strategically selected subset of core points can produce a reasonable $\varepsilon$-neighborhood cover for clustering.}
\vspace{-10pt}
\end{figure}

\subsection{Uniform Initialization}

DBSCAN++, shown in Algorithm~\ref{fig:dbscanpp}, proceeds as follows: First, it chooses a subset $S$ of $m$ uniformly sampled points from the dataset $X$. Then, it computes the empirical density of points in $S$ w.r.t. the entire dataset. That is, a point $x\in S$ is a core point if $|B(x, \varepsilon) \cap X| \ge \text{minPts}$. From here, DBSCAN++ builds a similar neighborhood graph $G$ of core-points in $S$ and finds the connected components in $G$. Finally, it clusters the rest of the unlabeled points to their closest core-points. Thus, since it only recovers a fraction of the core-points, it requires expensive density estimation queries on only $m$ of the $n$ samples. However, the intuition, as shown in Figure~\ref{fig:coverage}, is that a smaller sample of core-points can still provide adequate coverage of the dataset and lead to a reasonable clustering.

\begin{algorithm}
\caption{\label{fig:dbscanpp}DBSCAN++}
\begin{algorithmic}
\STATE {\bf Inputs:} $X$, $m$, $\varepsilon$, minPts
\STATE $\textit{S} \gets$ sample $m$ points from $X$.
\STATE $\textit{C} \gets$ all core-points in $S$ w.r.t $X$, $\varepsilon$, and minPts
\STATE $\textit{G} \gets$ empty graph.
\FOR {$c \in C$}
    \STATE Add an edge (and possibly a vertex or vertices) in $G$ from $c$ to all points in $X \cap B(c, \varepsilon)$
\ENDFOR
\STATE {\bf return} connected components of $G$.
\end{algorithmic}
\end{algorithm}

\subsection{K-Center Initialization}
Instead of uniformly choosing the subset of $m$ points at random, we can use $K$-center \citep{gonzalez1985clustering,har2011geometric}, which aims at finding the subset of size $m$ that minimizes the maximum distance of any point in $X$ to its closest point in that subset. In other words, it tries to find the most efficient covering of the sample points.
We use the greedy initialization method for approximating $K$-center (Algorithm~\ref{fig:kcenter}), which repeatedly picks the farthest point from any point currently in the set. This process continues until we have a total of $m$ points. This method gives a $2$-approximation to the $K$-center problem.

\begin{algorithm}[H]
\caption{\label{fig:kcenter}Greedy $K$-center Initialization}
\begin{algorithmic}
\STATE {\bf Input:} $X$, $m$.
\STATE $S \gets \{x_1 \}$.
\FOR {$i$ from $1$ to $m-1$}
    \STATE Add $\argmax_{x \in X}\min_{s \in S} |x - s|$ to $S$.
\ENDFOR
\STATE {\bf return} $S$.
\end{algorithmic}
\end{algorithm}

\subsection{Time Complexity}

DBSCAN++ has a time complexity of $O(nm)$. Choosing the set $S$ takes linear time for the uniform initialization method and $O(mn)$ for the greedy $K$-center approach \cite{gonzalez1985clustering}. The next step is to find the core-points. We use a KDTree to query for the points within the $\varepsilon$-radii ball for each point in $S$. Each such query takes $O(n)$ in the worst case, and doing so for $m$ sampled points leads to a cost of $O(nm)$. Constructing the graph takes $O(mn)$ time and running a depth-first search on the graph recovers the connected components in $O(nm)$ since the graph will have at most $O(nm)$ edges.

The last step is to cluster the remaining points to the nearest core point. We once again use a KDTree, which takes $O(m)$ for each of $O(n)$ points, leading to a time complexity of $O(nm)$ as well. Thus, the time complexity of DBSCAN++ is $O(nm)$.


%% file: Theory.tex
In this section, we show that DBSCAN++ is a consistent estimator of the density level-sets. It was recently shown by \citet{jiang2017density} that DBSCAN does this with finite-sample guarantees. We extend this analysis to show that our modified DBSCAN++ also has statistical consistency guarantees, and we show the trade-off between speed and convergence rate. 

\begin{definition}(Level-set)
The $\lambda$-level-set of $f$ is defined as $L_f(\lambda) := \{ x\in\mathcal{X} : f(x) \ge \lambda\}$.
\end{definition}
Our results are under the setting that the density level $\lambda$ is known and gives insight into how to tune the parameters based on the desired density level.

\subsection{Regularity Assumptions}

 We have $n$ i.i.d. samples $X = \{x_1,...,x_n\}$ drawn from a distribution $\mathcal{F}$ over $\mathbb{R}^D$. We take $f$ to be the density of $\mathcal{F}$ over the uniform measure on $\mathbb{R}^D$.
 \begin{assumption}\label{assumption1}
 $f$ is continuous and has compact support $\mathcal{X} \subseteq \mathbb{R}^D$.
 \end{assumption}

Much of the results will depend on the behavior of level-set boundaries. Thus, we require sufficient drop-off at the boundaries as well as separation between the CCs at a particular level-set.

Define the following shorthands for distance from a point to a set and the neighborhood around a set.
\begin{definition}
$d(x, A) := \inf_{x' \in A} |x - x'|$,
$B(C, r) := \{x \in \mathcal{X} : d(x, C) \le r \}$,
\end{definition}

\begin{assumption} [$\beta$-regularity of level-sets] \label{assumption2}
Let $0 < \beta < \infty$.
There exist $\check{C}, \hat{C}, r_c > 0$ such that the following holds for all $x \in B(L_f(\lambda), r_c) \backslash L_f(\lambda)$.
\begin{align*}
    \check{C} \cdot d(x, L_f(\lambda))^{\beta} \le \lambda - f(x) \le \hat{C} \cdot d(x, L_f(\lambda))^{\beta}.
\end{align*}
\end{assumption}

\begin{remark}
We can choose any $0 < \beta < \infty$. The $\beta$-regularity condition is a standard assumption in level-set analyses. See \cite{singh2009adaptive}.
The higher the $\beta$, the more smooth the density is around the boundary of the level-set and thus the less salient it is. This makes it more difficult to recover the level-set.
\end{remark}

\subsection{Hyperparameter Settings}

In this section, we state the hyperparameter settings in terms of $n$, the sample size, and the desired density level $\lambda$ in order for statistical consistency guarantees to hold. 
Define $C_{\delta, n} = 16\log(2/\delta)\sqrt{\log n}$, where $\delta$, $0 < \delta < 1$, is a confidence parameter which will be used later (i.e. guarantees will hold with probability at  least  $1 - \delta$).
\begin{align*}
\varepsilon = \left(\frac{\text{minPts}}{n\cdot v_D \cdot (\lambda - \lambda \cdot C_{\delta, n}^2/\sqrt{\text{minPts}})} \right)^{1/D},
\end{align*}
where $v_D$ is  the volume of the unit ball in $\mathbb{R}^D$ and $\text{minPts}$ satisfies
\begin{align*}
C_l \cdot (\log n)^2 \le \text{minPts} \le C_u \cdot (\log n)^{\frac{2D}{2+D}} \cdot n^{2\beta / (2\beta + D)},
\end{align*}
and $C_l$ and $C_u$ are positive constants depending on $\delta, f$.

\subsection{Level-set estimation  result}

We give the estimation rate under the Hausdorff metric. 
\begin{definition} [Hausdorff Distance]
\begin{align*}
d_{\text{Haus}}(A, A') = \max \{ \sup_{x \in A} d(x, A'), \sup_{x' \in A'} d(x', A) \}.
\end{align*}
\end{definition}
\begin{theorem} \label{theo:levelset}
Suppose Assumption~\ref{assumption1} and~\ref{assumption2} hold, and assume the parameter settings in the previous section.
There exists $C_l, C$ sufficiently large and $C_u$ sufficiently small such that the following holds with probability at least $1 - \delta$.
Let $\widehat{L_f(\lambda)}$ be the core-points returned by Algorithm~\ref{fig:dbscanpp} under uniform initialization or greedy $K$-center initialization. Then,
\begin{align*}
&d_{\text{Haus}}(\widehat{L_f(\lambda)}, L_f(\lambda)) \\
&\le C\cdot \left( C_{\delta, n}^{2/\beta} \cdot {\text{minPts}}^{-1/2\beta} + C_{\delta, n}^{1/D} \cdot \left(\frac{\sqrt{\log m}}{m}\right)^{1/D} \right).
\end{align*}
\end{theorem}

\begin{proof}
There are two quantities to bound: (i) $\max_{x \in \widehat{L_f(\lambda)}} d (x,  L_f(\lambda))$, which ensures that the estimated core-points are not far from the true core-points (i.e. $L_f(\lambda)$), and (ii) $\sup_{x \in L_f(\lambda)} d (x, \widehat{L_f(\lambda)})$, which ensures that the estimates core-points provides a good covering of the level-set. 

The bound for (i) follows by the main result of \citet{jiang2017density}.  This is because DBSCAN++'s estimated core-points are a subset of that of the original DBSCAN  procedure. Thus, $\max_{x \in \widehat{L_f(\lambda)}} d (x,  L_f(\lambda)) \le \max_{x \in \widetilde{L_f(\lambda)}} d (x,  L_f(\lambda))$, where $\widetilde{L_f(\lambda)}$ are the core-points returned by original DBSCAN; this quantity is bounded by $O(C_{\delta, n}^{2/\beta} \cdot \text{minPts}^{-1/2\beta})$ by \citet{jiang2017density}.

We now turn to the other direction and bound $\sup_{x \in L_f(\lambda)} d (x, \widehat{L_f(\lambda)})$. Let $x \in L_f(\lambda)$.

Suppose we use the uniform initialization.
Define $r_0 := \left(\frac{2C_{\delta, n} \sqrt{D \log m}}{m v_D\cdot  \lambda }\right)^{1/D}$. Then, we have
\begin{align*}
&\int_{\mathcal{X}} f(z) \cdot  1[z \in B(x, r_0)] dz \ge v_D{r_0}^D(\lambda - \hat{C} r_0^\beta) \\
&\ge  v_D{r_0}^D \lambda /2 
= \frac{C_{\delta, n} \sqrt{D\log m}}{m},
\end{align*}
where the first inequality holds from Assumption~\ref{assumption2}, the second inequality holds for $n$ sufficiently large, and the last holds from the conditions on $\text{minPts}$.

By the uniform ball convergence rates of Lemma 7 of \citet{chaudhuri2010rates}, we have that with high probability, there 
exists sample point $x' \in S$ such that  $|x - x'| \le r_0$. This is because the ball $B(x, r_0)$ contains sufficiently high true mass to be guaranteed a sample point in $S$. Moreover, this guarantee holds with high probability uniformly over $x \in \mathcal{X}$.
Next, we show that $x'$ is a core-point. This follows by Lemma 8 of \citet{jiang2017density}, which shows that any sample point in $x \in L_f(\lambda)$ satisfies $|B(x,\varepsilon) \cap X| \ge \text{minPts}$. Thus, $x'  \in \widehat{L_f(\lambda)}$. Hence, $\sup_{x \in L_f(\lambda)} d (x, \widehat{L_f(\lambda)}) \le r_0$, as desired.
 
Now suppose we use the greedy $K$-center initialization. Define the following attained $K$-center objective: 
\begin{align*}
    \tau := \max_{x \in X} \min_{s \in S} d(s, x),
\end{align*}
and the optimal $K$-center objective:
\begin{align*}
    \tau_{\text{opt}} := \min_{S' \subseteq \mathcal{X}, |S'| = m} \max_{x \in X} \min_{s \in S'} d(s, x).
\end{align*}
It is known that the greedy $K$-center initialization is a $2$-approximation
(see \citet{gonzalez1985clustering,har2011geometric}), thus
\begin{align*}
   \tau \le 2 \tau_{\text{opt}} \le 2r_0,
\end{align*}
where the last inequality follows with high probability since the $K$-center objective will be sub-optimal if we sampled the $m$ centers uniformly.
Then, we have
\begin{align*}
   & \sup_{x \in L_f(\lambda)} \min_{s \in S} d(s, x) \\
    &\le 
    \max_{x \in X} \min_{s \in S} d(s, x) + d_{\text{Haus}}(L_f(\lambda), X \cap L_f(\lambda))
    \\ &\le \tau + r_0 \le 3r_0.
\end{align*}
The argument then proceeds in the same way as with uniform initialization but with an extra constant factor, as desired.
\end{proof}

\begin{remark}
When taking $\text{minPts}$ to the maximum allowed rate
\begin{align*}
    \text{minPts} \approx n^{2\beta/(2\beta + D)},
\end{align*}
we obtain the error rate (ignoring log factors) of 
\begin{align*}
    d_{\text{Haus}}(\widehat{L_f(\lambda)}, L_f(\lambda)) \lesssim n^{-1/(2\beta +D)} + m^{-1/D}.
\end{align*}

The first term matches the known lower bound for level-set estimation established in Theorem 4 of \citet{tsybakov1997nonparametric}. The second term is the trade-off for computing the empirical densities for only $m$ of the points. In particular, if we take 
\begin{align*}
    m \gtrsim  n^{D/(2\beta + D)},
\end{align*}
then the first term dominates, and we thus have $d_{\text{Haus}}(\widehat{L_f(\lambda)}, L_f(\lambda))\lesssim n^{-1/(2\beta +D)}$, the minimax optimal rate for level-set estimation. This leads to the following result.
\end{remark}

\begin{corollary}
Suppose the conditions of Theorem~\ref{theo:levelset} and set $m \approx n^{D/(2\beta + D)}$. Then, Algorithm~\ref{fig:dbscanpp} is a minimax optimal estimator (up to logarithmic factors) of the density level-set with sub-quadratic runtime of $O(n^{2 - 2\beta/(2\beta + D)})$.
\end{corollary}

\subsection{Estimating the connected components}

The previous section shows that the core-points returned by DBSCAN++ recovers the density level-set. The more interesting question is about the actual clustering: that is, whether DBSCAN++ can recover the connected components of the density level-set individually and whether there is a 1:1 correspondence between the clusters returned by DBSCAN++ and the connected components. 

It turns out that to obtain such a result, we need a minor modification of the procedure. That is, after determining the core points, instead of using the $\varepsilon$ cutoff to connect points into the same cluster, we must use a higher cutoff. In fact, any constant value would do as long as it is sufficiently smaller than the pairwise distances between the connected components. For example, for the original DBSCAN algorithm, many analyses must make this same modification. This is known as {\it pruning false clusters} in the literature (see  \citet{kpotufe2011pruning,jiang2017density}). The same analysis carries over to our modification, and we omit it here. We note that pruning does not change the final estimation rates but may change the initial sample size required. 

\begin{figure}
    \begin{center}
        \begin{tabular}{ |p{2.8cm}||p{0.8cm}|p{0.4cm}|p{0.3cm}|p{0.9cm}|p{0.9cm}| }
             \hline
             & $n$ & $D$ & $c$ & $m$\\
             \hline\hline
             (A) iris & 150 & 4 & 3 & 3\\
             \hline
             (B) wine & 178 & 13 & 3 & 5\\
             \hline
             (C) spam & 1401 & 57 & 2 & 793\\
             \hline 
             (D) images & 210 & 19 & 7 & 24\\
             \hline
             (E) MNIST & 60000 & 20 & 10 & 958\\
             \hline
             (F) Libras & 360 & 90 & 15 & 84\\
             \hline
             (G) mobile & 2000 & 20 & 4 & 112\\
             \hline
             (H) zoo & 101 & 16 & 7 & 8\\
             \hline
             (I) seeds & 210 & 19 & 7 & 6\\
             \hline
             (J) letters & 20000 & 16 & 26 & 551\\
             \hline
             (K) phonemes & 4509 & 256 & 5 & 396\\
             \hline
             (L) fashion MNIST & 60000 & 784 & 10 & 5674 \\
             \hline
             (M) celeb-a & 10000 & 40 & 3 & 3511 \\
             \hline
        \end{tabular}
    \caption{\label{fig:datasetsummary}\textit{Summary of datasets used.} Includes dataset size (n), number of features ($D$), number of clusters ($c$), and the ($m$) used by both DBSCAN++ uniform and $K$-center. }
    \end{center}
    \vspace{-0.4cm}
\end{figure}

\subsection{Outlier detection}

One important application of DBSCAN is outlier detection \cite{breunig2000lof,ccelik2011anomaly,thang2011anomaly}. Datapoints not assigned to clusters are {\it noise points} and can be considered outliers. This is because the noise points are the low density points away from the clusters and thus tend to be points with few similar representatives in the dataset.  We show that the noise points DBSCAN++ finds are similar to the noise points discovered by DBSCAN++. We give a simple result that shows that  every DBSCAN noise point is also one DBSCAN++ finds (Lemma~\ref{lemma:noisepoints}). Then, Figure~\ref{figure:tradeoff} (Left) shows that the number of noise points of DBSCAN++ quickly converges to those of DBSCAN as the ratio $m/n$ increases, which combined with Lemma~\ref{lemma:noisepoints}, shows that the noise points DBSCAN++ returns  closely approximates those returned by DBSCAN for $m/n$ sufficiently high.

\begin{figure}
\includegraphics[width=0.5\textwidth]{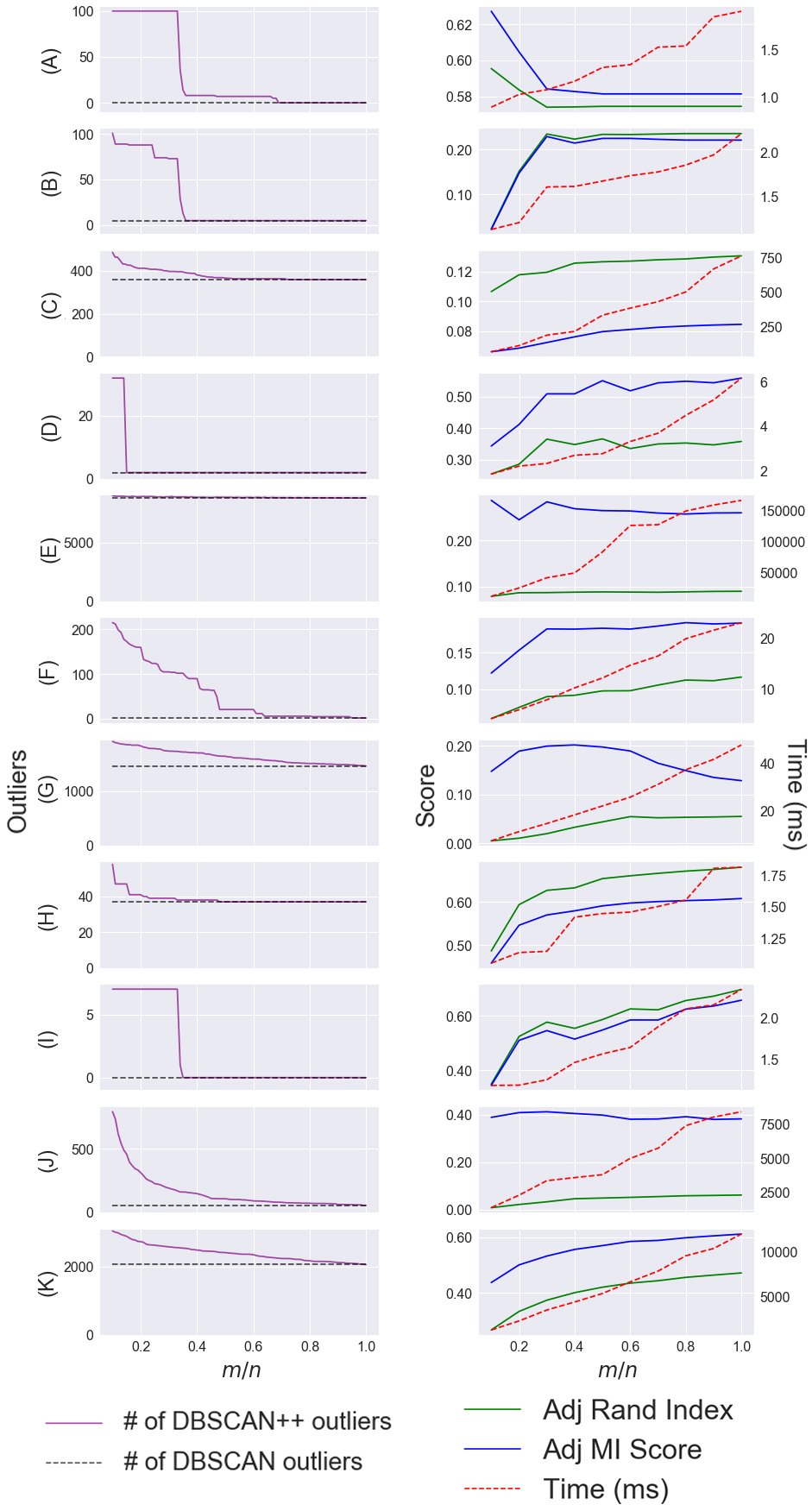}
\caption{\label{figure:tradeoff} Each row corresponds to a dataset. See Figure~\ref{fig:datasetsummary} for dataset descriptions. {\bf Left (Outlier Detection)}: The number of outliers (i.e. noise points) returned by DBSCAN against $m/n$ and compared it to that of DBSCAN++. We see that the set of DBSCAN++ outliers quickly approaches those of DBSCAN's thus showing that DBSCAN++ remains effective at outlier detection compared to DBSCAN, especially when $m/n$ is sufficiently high. {\bf Right (Clustering Performance)}: we plot the clustering accuracy and runtimes for eight real datasets as a function of the ratio $m/n$. As expected, the runtime increases approximately linearly in this ratio, but the clustering scores consistently attain high values when $m/n$ is sufficiently large. Interestingly, sometimes we attain higher scores with lower values of $m/n$ thus giving both better runtime and accuracy. }
\end{figure}

\begin{lemma}[Noise points]\label{lemma:noisepoints}
For any dataset, if $N_0$ and $N_1$ are the noise points found by DBSCAN and DBSCAN++ respectively, then as long as they have the same setting of $\varepsilon$ and $k$, we have that $N_0 \subseteq N_1$. 
\end{lemma}
\vspace{-0.4cm}
\begin{proof}
Noise points are those that are further than $\varepsilon$ distance away from a core point. The result follows since DBSCAN++ core points are a subset of that of DBSCAN.
\end{proof}

%% file: Experiments.tex
\subsection{Dataset and setup}

We ran DBSCAN++ with uniform and $K$-center initializations and compared both to DBSCAN on $11$ real datasets as described in Figure~\ref{fig:datasetsummary}. 
We used Phonemes \citep{friedman2001elements}, a dataset of log periodograms of spoken phonemes, and MNIST, a sub-sample of the MNIST handwriting recognition dataset after running a PCA down to $20$ dimensions. The rest of the datasets we used are standard UCI or Kaggle datasets used for clustering. We evaluate the performance via two widely-used clustering scores: the adjusted RAND index \citep{hubert1985comparing} and adjusted mutual information score \citep{vinh2010information}, which are computed against the ground truth. We fixed $\text{minPts} = 10$ for all procedures throughout experiments.

\subsection{Trade-off between accuracy and speed}

The theoretical results suggest that up to a certain point, only computing empirical densities for a subset of the sample points should not noticeably impact the clustering performance. Past that point, we begin to see a trade-off.
We confirm this empirically in Figure~\ref{figure:tradeoff} (Right), which shows that indeed past a certain threshold of $m/n$, the clustering scores remain high. Only when the sub-sample is too small do we begin seeing a significant trade-off in clustering scores. This shows that DBSCAN++ can save considerable computational cost while maintaining similar clustering performance as DBSCAN.

We further demonstrate this point by applying these procedures to image segmentation, where segmentation is done by clustering the image's pixels (with each pixel represented as a $5$-dimensional vector consisting of $(x,y)$ position and RGB color).  Figure~\ref{figure:imagesegmentation} shows that DBSCAN++ provides a very similar segmentation as DBSCAN in a fraction of the time. While this is just a simple qualitative example, it serves to show the applicability of DBSCAN++ to a potentially wide range of problems.

\begin{figure}[H]
\includegraphics[width=0.48\textwidth]{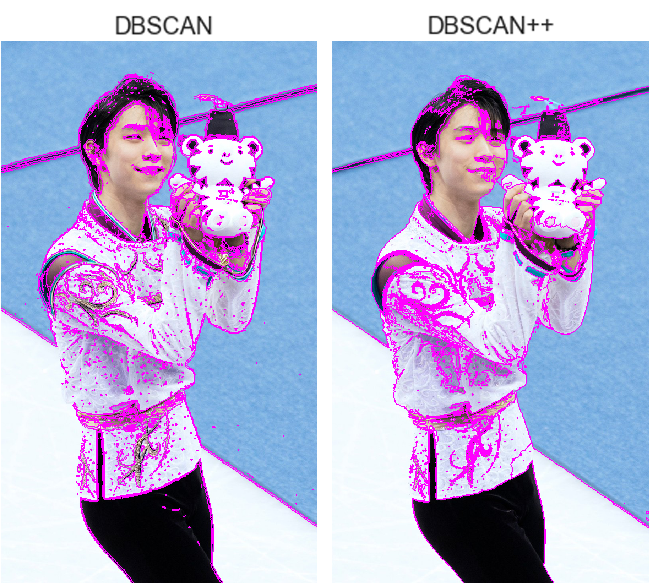}
\vspace{-0.5cm}
\caption{\label{figure:imagesegmentation}\textit{Figure skater Yuzuru Hanyu at the 2018 Olympics.} DBSCAN was initiated with hyperparameters $\varepsilon=8$ and minPts $= 10$, and DBSCAN++  with $\varepsilon=60$, $m/n=0.3$, and minPts $= 10$. DBSCAN++ with $K$-centers initialization recovers similar clusters in the $988\times 750$ image as DBSCAN in far less time: 7.38 seconds versus 44.18 seconds. The speedup becomes more significant on higher resolution images.}
\vspace{-0.3cm}
\end{figure}

\subsection{Robustness to Hyperparameters}

In Figure~\ref{figure:experiments}, we plot each algorithm's performance across a wide range of its hyperparameters. The table in Figure~\ref{fig:results} shows the best scores and runtimes for each dataset and algorithm. For these experiments, we chose $m = p \cdot n^{D / (D + 4)}$, where $0 < p < 1$ was chosen based on validation over just $3$ values, as explained in Figure~\ref{fig:results}. We found that the $K$-center initialization required smaller $p$ due to its ability to find a good covering of the space and more efficiently choose the sample points to query for the empirical density.


The results in Figure~\ref{figure:experiments} show that DBSCAN++ with uniform initialization gives competitive performance compared to DBSCAN but with robustness across a much wider range of $\epsilon$. In fact, in a number of cases, DBSCAN++ was even better than DBSCAN under optimal tuning. 
DBSCAN++ with $K$-center initialization further improves on the clustering results of DBSCAN++ for most of the datasets. Pruning the core-points as DBSCAN++ may act as a regularizing factor to prevent the algorithm's dependency on the preciseness of its parameters.

An explanation of why DBSCAN++ added robustness across $\varepsilon$ follows. When tuning DBSCAN with respect to $\varepsilon$, we found that DBSCAN often performed optimally on only a narrow range of $\varepsilon$. Because $\varepsilon$ controls both the designation of points as core-points as well as the connectivity of the core-points, small changes could produce significantly different clusterings.

\begin{figure}[H]
\includegraphics[width=0.48\textwidth]{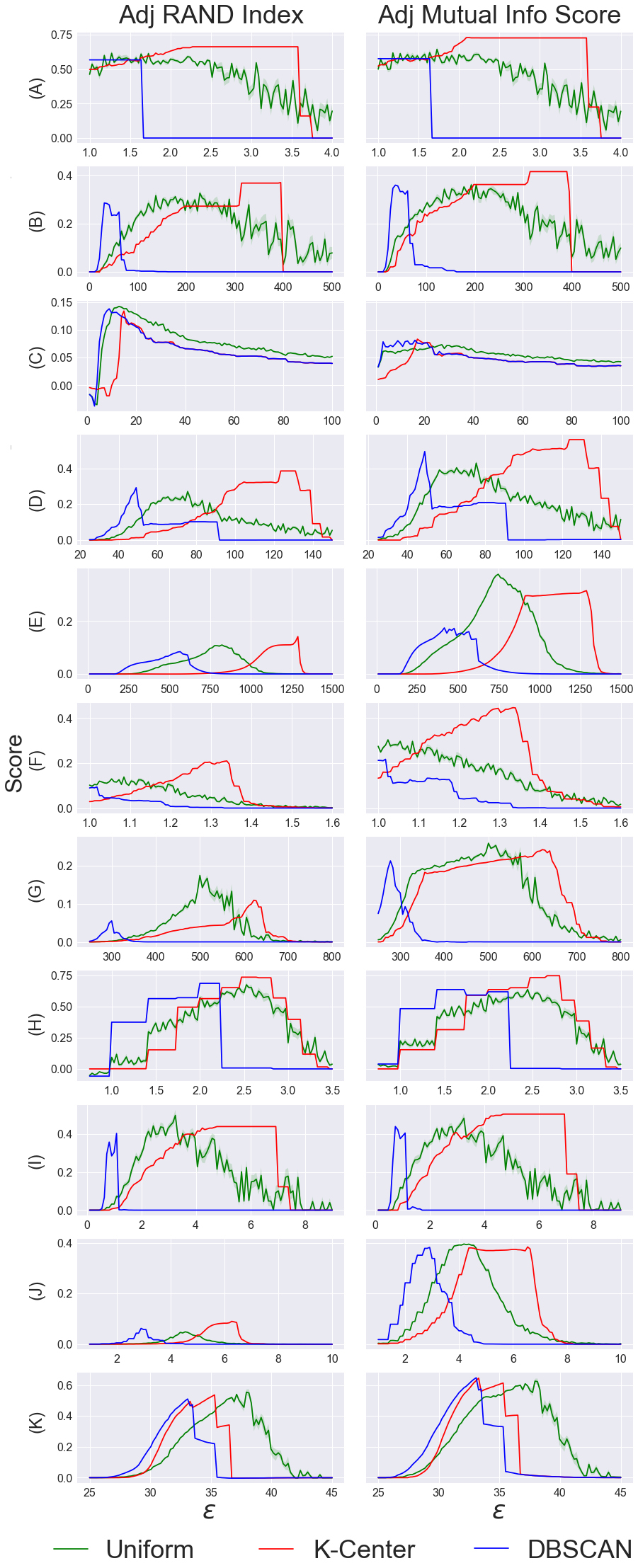}
\vspace{-0.5cm}
\caption{\label{figure:experiments} {\bf Clustering performance over range of hyperparameter  settings}. Experimental results on datasets described in Figure~\ref{fig:datasetsummary}.  Each row corresponds to a single dataset and each column corresponds to a clustering score. For each dataset and clustering score, we plot the scores for DBSCAN++ with uniform and $K$-center sampling vs DBSCAN across a wide range of settings for $\varepsilon$ ($x$-axis).}
\end{figure}

\begin{figure}[H]
\begin{center}
\begin{tabular}{|p{0.33cm}||p{1.3cm}|p{2.2cm}|p{1.4cm}|}
\hline
 & DBSCAN & Uniform & K-Center \\ \hline\hline
 (A) & 0.5681  & {\bf \textcolor{darkorchid}{0.6163}} ($\pm$0.01)  & {\bf \textcolor{darkturquoise}{0.6634}} \\
 & 0.5768  & {\bf \textcolor{darkorchid}{0.6449}} ($\pm$0.01)  & {\bf \textcolor{darkturquoise}{0.7301}} \\ \hline
 (B) & 0.2851  & {\bf \textcolor{darkorchid}{0.3254}} ($\pm$0.01)  & {\bf \textcolor{darkturquoise}{0.3694}} \\
 & 0.3587  & {\bf \textcolor{darkorchid}{0.3605}} ($\pm$0.00)  & {\bf \textcolor{darkturquoise}{0.4148}} \\ \hline
 (C) & 0.2851  & {\bf \textcolor{darkorchid}{0.3254}} ($\pm$0.01)  & {\bf \textcolor{darkturquoise}{0.3694}} \\
 & 0.3587  & {\bf \textcolor{darkorchid}{0.3605}} ($\pm$0.00)  & {\bf \textcolor{darkturquoise}{0.4148}} \\ \hline
 (D) & {\bf \textcolor{darkorchid}{0.2922}}  & 0.2701 ($\pm$0.01)  & {\bf \textcolor{darkturquoise}{0.3853}} \\
 & {\bf \textcolor{darkorchid}{0.4938}}  & 0.4289 ($\pm$0.01)  & {\bf \textcolor{darkturquoise}{0.5600}} \\ \hline
 (E) & 0.0844  & {\bf \textcolor{darkorchid}{0.1097}} ($\pm$0.00)  & {\bf \textcolor{darkturquoise}{0.1416}} \\
 & 0.1743  & {\bf \textcolor{darkturquoise}{0.3774}} ($\pm$0.00)  & {\bf \textcolor{darkorchid}{0.3152}} \\ \hline
 (F) & 0.0939  & {\bf \textcolor{darkorchid}{0.1380}} ($\pm$0.00)  & {\bf \textcolor{darkturquoise}{0.2095}} \\
 & 0.2170  & {\bf \textcolor{darkorchid}{0.3033}} ($\pm$0.00)  & {\bf \textcolor{darkturquoise}{0.4461}} \\ \hline
 (G) & 0.0551  & {\bf \textcolor{darkturquoise}{0.1741}} ($\pm$0.00)  & {\bf \textcolor{darkorchid}{0.1091}} \\
 & 0.2123  & {\bf \textcolor{darkturquoise}{0.2585}} ($\pm$0.00)  & {\bf \textcolor{darkorchid}{0.2418}} \\ \hline
 (H) & {\bf \textcolor{darkorchid}{0.6846}} & 0.6729 ($\pm$0.01)  & {\bf \textcolor{darkturquoise}{0.7340}} \\
 & 0.6347  & {\bf \textcolor{darkorchid}{0.6356}} ($\pm$0.00)  & {\bf \textcolor{darkturquoise}{0.7456}} \\ \hline
 (I) & 0.4041  & {\bf \textcolor{darkturquoise}{0.4991}} ($\pm$0.02)  & {\bf \textcolor{darkorchid}{0.4402}} \\
 & 0.4403  & {\bf \textcolor{darkorchid}{0.4843}} ($\pm$0.02)  & {\bf \textcolor{darkturquoise}{0.5057}} \\ \hline
 (J) & {\bf \textcolor{darkorchid}{0.0623}}  & 0.0488 ($\pm$0.00)  & {\bf \textcolor{darkturquoise}{0.0901}} \\
 & 0.3823  & {\bf \textcolor{darkturquoise}{0.3956}} ($\pm$0.00)  & {\bf \textcolor{darkorchid}{0.3841}} \\ \hline
 (K) & 0.5101  & {\bf \textcolor{darkturquoise}{0.5541}} ($\pm$0.01)  & {\bf \textcolor{darkorchid}{0.5364}} \\
 & {\bf \textcolor{darkturquoise}{0.6475}}  & 0.6259 ($\pm$0.01)  & {\bf \textcolor{darkorchid}{0.6452}} \\\hline
\end{tabular}
\caption{\label{fig:results}{\bf Highest scores for each dataset and algorithm.} The first row is the adjusted RAND index and the second row the adjusted mutual information. The highest score of the row is in {\bf \textcolor{darkturquoise}{green}} while the second highest is in {\bf \textcolor{darkorchid}{orange}}. The standard error over 10 runs is given in parentheses for DBSCAN++ with uniform initialization. Both other algorithms are deterministic. Each algorithm was tuned across a range of $\epsilon$ with minPts $= 10$. For both DBSCAN++ algorithms, we used $p$ values of $0.1, 0.2,$ or $0.3$. We see that DBSCAN++ uniform performs better than DBSCAN on 17 out of 22 metrics, while DBSCAN++ $K$-center performs better than DBSCAN on 21 out of 22 metrics.}
\end{center}
\vspace{-0.5cm}
\end{figure}

\begin{figure}
\begin{center}
\begin{tabular}{|p{0.33cm}||p{2.2cm}|p{2.2cm}|p{2.0cm}|}
\hline
 & DBSCAN & Uniform & K-Center \\ \hline\hline
 (A) & 3.07 ($\pm$0.08) & 1.52 ($\pm$0.09) & 2.55 ($\pm$0.34) \\ \hline
 (B) & 2.04 ($\pm$0.07) & 1.31 ($\pm$0.07) & 0.79 ($\pm$0.02) \\ \hline
 (C) & 3308 ($\pm$26.4) & 225.86 ($\pm$6.8) & 442.69 ($\pm$2.0) \\ \hline
 (D) & 4.88 ($\pm$0.09) & 1.51 ($\pm$0.05) & 1.32 ($\pm$0.04) \\ \hline
 (E) & 1.5e5 ($\pm$0.17) & 3.5e3 ($\pm$39.23) & 7.0e3 ($\pm$41.1) \\ \hline
 (F) & 37.63 ($\pm$0.38) & 8.20 ($\pm$0.22) & 9.84 ($\pm$0.06) \\ \hline
 (G) & 67.05 ($\pm$0.63 & 11.41 ($\pm$0.21) & 15.23 ($\pm$0.32) \\ \hline
 (H) & 1.07 ($\pm$0.03) & 0.78 ($\pm$0.03) & 0.81 ($\pm$0.03) \\ \hline
 (I) & 1.75 ($\pm$0.04) & 0.91 ($\pm$0.03) & 0.97 ($\pm$0.09) \\ \hline
 (J) & 1.0e5 ($\pm$76.43) & 5.2e3 ($\pm$17.48) & 1.5e3 ($\pm$36.4) \\ \hline
 (K) & 1.2e4 ($\pm$160) & 1.9e3 ($\pm$32.45) & 1.9e3 ($\pm$30.4) \\ \hline
 (L) & 3.9e9 ($\pm$4.3e4) &  7.4e8 ($\pm$4.1e3) &  3.6e8($\pm$307) \\ \hline
 (M) & 4.1e9 ($\pm$6.2e4) &  3.1e8 ($\pm$411) &  2.3e8($\pm$1.1e3) \\ \hline
\end{tabular}
\caption{\label{fig:results_runtime}  {\bf Runtimes} (milliseconds) and standard errors for each dataset and algorithm. DBSCAN++ using both uniform and $K$-center initializations performs reasonably well within a fraction of the runtime of DBSCAN. The larger the dataset, the less time DBSCAN++ requires compared to DBSCAN, showing that DBSCAN++ scales much better in practice.}
\end{center}
\vspace{-0.5cm}
\end{figure}

In contrast, DBSCAN++ suffers less from the hyper-connectivity of the core-points until $\varepsilon$ is very large. It turns out that only processing a subset of the core-points not only reduces the runtime of the algorithm, but it provides the practical benefit of reducing the tenuous connections between connected components that are actually far away. This way, DBSCAN++ is much less sensitive to changes in $\varepsilon$ and reaches its saturation point (where there is only one cluster) only at very large $\varepsilon$. 

Performance under optimal tuning is often not available in practice, and this is especially the case in unsupervised problems like clustering where the ground truth is not assumed to be known. Thus, not only should procedures produce accurate clusterings in the best setting, but it may be even more important for procedures to be precise, {\it easy to tune}, reasonable across a {\it wide range} of its hyperparameter settings. This added robustness (not to mention speedup) may make DBSCAN++ a more practical method. This is especially true on large datasets where it may be costly to iterate over many hyperparameter settings. 

\subsection{Performance under optimal tuning}
We see that under optimal tuning of each algorithm, DBSCAN++ consistently outperforms DBSCAN in both clustering scores and runtime. We see in Figure~\ref{fig:results} that DBSCAN++ with the uniform initialization consistently outperforms DBSCAN and DBSCAN++ with $K$-center initialization consistently outperforms both of the algorithms. Figure~\ref{fig:results_runtime} shows that indeed DBSCAN++ gives a speed advantage over DBSCAN for the runs that attained the optimal performance. These results thus suggest that not only is DBSCAN++ faster, it can achieve better clusterings. 

%% file: Conclusion.tex
In this paper, we presented DBSCAN++, a modified version of DBSCAN that only computes the density estimates for a subset $m$ of the $n$ points in the original dataset. We established statistical consistency guarantees which show the trade-off between computational cost and estimation rates, and we prove that interestingly, up to a certain point, we can enjoy the same estimation rates while reducing computation cost. We also demonstrate this finding empirically. We then showed empirically that not only can DBSCAN++ scale considerably better than DBSCAN,  its performance is competitive in accuracy and consistently more robust across their mutual bandwidth hyperparameters. Such robustness can be highly desirable in practice where optimal tuning is costly or unavailable.

%% file: Appendix.tex
\section{Comparison to other methods}
In this section, we compare DBSCAN++ against replacing the nearest-neighbor search needed for DBSCAN with an approximate nearest neighbor method using the FLANN (\hyperlink{https://www.cs.ubc.ca/research/flann/}{https://www.cs.ubc.ca/research/flann/}) library, and we call it ANN DBSCAN.
\begin{figure}[H]
\includegraphics[width=0.98\textwidth]{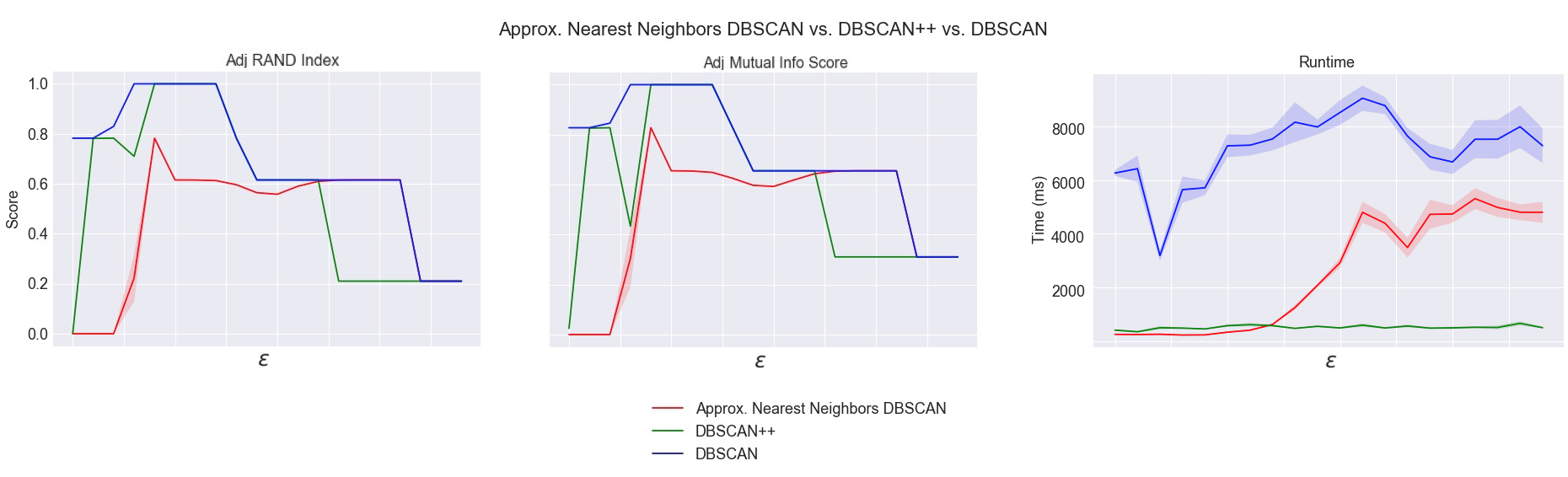}
\caption{\label{figure:ann} {\bf DBSCAN using approximate nearest neighbors vs. DBSCAN++ vs. DBSCAN}. Experimental results on a synthetic dataset of 10,000 points drawn from five 50-dimensional uniform distributions run on DBSCAN++, DBSCAN, and DBSCAN using a fast approximate nearest neighbors algorithm from the FLANN library. DBSCAN++ was run with $K$-center initialization and $m/n=0.1$. All algorithms were run with $minPts = 10$. ANN DBSCAN shows a comparable speedup to DBSCAN++ but poorer performance compared to both DBSCAN and DBSCAN++, whereas DBSCAN++ shows both comparable performance to DBSCAN and comparable runtime to ANN DBSCAN.}
\end{figure}